\documentclass[conference, letterpaper]{IEEEtran}
\usepackage[utf8]{inputenc} %
\usepackage[T1]{fontenc}    %
\usepackage[
    dvipsnames
]{xcolor}                   %
\usepackage[
   colorlinks=true,
   citecolor=MidnightBlue,
   linkcolor=MidnightBlue,
   filecolor=MidnightBlue,      
   urlcolor=MidnightBlue
]{hyperref}                 %
\usepackage{url}            %
\usepackage{booktabs}       %
\usepackage{amsfonts}       %
\usepackage{nicefrac}       %
\usepackage{microtype}      %
\usepackage{amsmath}        %
\usepackage{amsthm}         %
\usepackage{amssymb}        %
\usepackage{amsfonts}       %
\usepackage{mathtools}      %
\usepackage{bm, bbm}        %
\usepackage{thm-restate}    %
\usepackage[
    scr=boondoxo,
    scrscaled=1.05
]{mathalpha}                %
\usepackage{csquotes}       %
\usepackage{comment}        %
\usepackage{dsfont}         %
\usepackage[
    square, 
    numbers, 
    compress,
    sort
]{natbib}                    %
\usepackage{enumitem}       %
\usepackage{blindtext}      %
\usepackage{tikz}           %
\usepackage{graphicx}                 %
\usepackage{tikzscale}      %
\usepackage[
    nameinlink
]{cleveref}                 %
\usepackage{caption}        %
\usepackage{subcaption}     %

\definecolor{orange}{rgb}{1,0.4,0.0}

\DeclarePairedDelimiterXPP{\KL}[2]{D_\textnormal{KL}}{(}{)}{}{%
#1\:\delimsize\|\:#2%
}
\DeclarePairedDelimiterXPP{\RD}[2]{D_{$\alpha$}}{(}{)}{}{%
#1\:\delimsize\|\:#2%
}

\DeclarePairedDelimiterXPP\Prob[1]{\mathbb{P}}{\lbrace}{\rbrace}{}{

#1}

\DeclarePairedDelimiterXPP{\lnorm}[2]{}{\lVert}{\rVert}{_{#2}}{#1}

\newcommand{\bA}{\ensuremath{\mathbb{A}}}

\newcommand{\bE}{\ensuremath{\mathbb{E}}}

\newcommand{\bN}{\ensuremath{\mathbb{N}}}

\newcommand{\bP}{\ensuremath{\mathbb{P}}}
\newcommand{\bQ}{\ensuremath{\mathbb{Q}}}
\newcommand{\bR}{\ensuremath{\mathbb{R}}}

\newcommand{\cA}{\ensuremath{\mathcal{A}}}
\newcommand{\cB}{\ensuremath{\mathcal{B}}}

\newcommand{\cE}{\ensuremath{\mathcal{E}}}

\newcommand{\cK}{\ensuremath{\mathcal{K}}}

\newcommand{\cO}{\ensuremath{\mathcal{O}}}
\newcommand{\cP}{\ensuremath{\mathcal{P}}}

\newcommand{\cS}{\ensuremath{\mathcal{S}}}

\newcommand{\cW}{\ensuremath{\mathcal{W}}}

\newcommand{\cZ}{\ensuremath{\mathcal{Z}}}

\newcommand{\ccR}{\ensuremath{\mathscr{R}}}

\newcommand{\rmd}{\ensuremath{\mathrm{d}}}

\newcommand{\minf}{\textup{I}}

\newcommand{\relent}{\textup{D}}
\newcommand{\relentber}{\textup{d}}
\newcommand{\comp}{\mathfrak{C}_{n,\beta,S}}

\newcommand{\poprisk}{\ccR}
\newcommand{\emprisk}{\widehat{\ccR}}

\newtheoremstyle{mytheoremstyle} %
    {\topsep}                    %
    {\topsep}                    %
    {\itshape}                   %
    {}                           %
    {\bf}                        %
    {.}                          %
    {.5em}                       %
    {}  %

\DeclareMathOperator*{\esssup}{ess\,sup}

\theoremstyle{mytheoremstyle}
\newtheorem{lemma}{Lemma}
\crefname{lemma}{Lemma}{Lemmata}

\newtheorem{theorem}{Theorem}

\usepackage{balance}

\title{A note on generalization bounds \\ for losses with finite moments}

\author{%
    \IEEEauthorblockN{Borja Rodríguez-Gálvez\IEEEauthorrefmark{1},
                      Omar Rivasplata\IEEEauthorrefmark{2},
                      Ragnar Thobaben\IEEEauthorrefmark{1},
                      Mikael Skoglund\IEEEauthorrefmark{1}}
    \IEEEauthorblockA{\IEEEauthorrefmark{1}%
        KTH Royal Institute of Technology, \{borjarg, ragnart, skoglund\}@kth.se}
    \IEEEauthorblockA{\IEEEauthorrefmark{2}%
        UCL, o.rivasplata@ucl.ac.uk}
}

\date{}

\begin{document}

\maketitle

\begin{abstract}
  This paper studies the truncation method from~\citet{alquier2006transductive} to derive high-probability PAC-Bayes bounds for unbounded losses with heavy tails. Assuming that the $p$-th moment is bounded, the resulting bounds interpolate between a slow rate $\nicefrac{1}{\sqrt{n}}$ when $p=2$, and a fast rate $\nicefrac{1}{n}$ when $p \to \infty$ and the loss is essentially bounded. Moreover, the paper derives a high-probability PAC-Bayes bound for losses with a bounded variance. This bound has an exponentially better dependence on the confidence parameter and the dependency measure than previous bounds in the literature. Finally, the paper extends all results to guarantees in expectation and single-draw PAC-Bayes. In order to so, it obtains analogues of the PAC-Bayes fast rate bound for bounded losses from~\citep{rodriguezgalvez2023pacbayes} in these settings.
\end{abstract}

\section{Introduction}
\label{sec:introduction}

\looseness=-1 Consider a sequence of $n$ instances $s = (z_1, \ldots, z_n) \in \cZ^n$ of a problem with instance space $\cZ$. A \emph{learning algorithm} $\bA$ is a (possibly randomized) mechanism that generates a hypothesis $w \in \cW$ of the solution of the problem when it is given the sequence $s$, which is commonly referred to as the the \emph{training set}. 
The performance of a hypothesis $w$ on an instance %
$z$ is evaluated by a loss function 
$\ell : \cW\times\cZ \to \bR_+$ 
so that smaller values of $\ell(w,z)$ indicate a better performance of the hypothesis $w$ on the problem instance $z$, while larger values indicate a worse performance. 
Assume the instances of the problem follow a distribution $\bP_Z$; the goal of the learning algorithm is to produce a hypothesis $w$ that has as low as possible expected loss on samples $Z$ from the %
distribution $\bP_Z$, that is, a small \emph{population risk} $\poprisk(w) \coloneqq \bE \ell(w,Z)$.

\looseness=-1 Often, we do not have a direct access to the problem distribution $\bP_Z$, and hence calculating the population risk is unfeasible. Nonetheless, we can employ the available training set $s$ to construct 
an estimate of the population risk and bound its deviation. 
A common estimate is the \emph{empirical risk} $\emprisk(w,s) \coloneqq \frac{1}{n} \sum_{i=1}^n \ell(w,z_i)$, which is the average loss of the hypothesis $w$ on the training set instances $z_i$. 
Notice that the population risk can be decomposed as $\poprisk(w) = \emprisk(w,s) +  \big( \poprisk(w) - \emprisk(w,s) \big)$, where the second term is usually referred to as the generalization gap.

\looseness=-1 \emph{Probably approximately correct} (PAC) theory studies bounds on the generalization gap that hold with a probability larger than a user-chosen threshold. Classically, these bounds hold uniformly for all elements of a hypothesis class $\cW$ and only depend on the complexity of the said class, which is measured, for example, by the Vapnik--Cherovenkis (VC) dimension or the Rademacher complexity. See~\cite{shalev2014understanding} for an introduction to the topic.

\looseness=-1 In this paper, we consider \emph{PAC-Bayesian bounds}~\citep{shawe1997pac,mcallester1998some,mcallester1999pac,mcallester2003pac}. This framework considers the algorithm as a Markov kernel $\bP_W^S$ that returns a distribution $\bP_W^{S=s}$ on the hypothesis class, for every dataset realization $s$. Then, the resulting bounds depend not only on the hypothesis class, but also on the dependence of the hypothesis $W = \bA(S)$ on the random training set $S$. We are interested in the case of unbounded losses.

\subsection{PAC-Bayesian bounds}
\label{subsec:pac-bayesian-bounds}

The original PAC-Bayesian bound of~\citet{mcallester1998some, mcallester1999pac, mcallester2003pac} 
assumes bounded losses $\ell(w,z) \in [0,1]$ and
states that if $\bQ_W$ is a distribution on $\cW$, independent of the training set $S$, and $\beta \in (0,1)$ is a confidence parameter, then, with probability no smaller than $1 - \beta$
over the random training set $S \sim \bP_S = \bP_Z^{\otimes n},$
\begin{equation}
    \label{eq:mc_allester}
    \bE^S \poprisk(W) \leq \bE^S \emprisk(W,S) + \sqrt{\frac{\relent(\bP_W^S \Vert \bQ_W) + \log \frac{\xi(n)}{\beta}}{2n}}
\end{equation}
holds \emph{simultaneously} $\forall\ \bP_W^S \in \cP$,
where $\xi(n) \in [\sqrt{n}, 2 + \sqrt{2n}]$~\citep{maurer2004note, germain2015risk, rodriguezgalvez2023pacbayes}, $\cP$ is the set of all Markov kernels $\bP_W^S$ from $\cS$ to distributions on $\cW$ such that $\bP_W^S \ll \bQ_W$,
and $\bE^S$ denotes the conditional expectation operator with respect to the $\sigma$-algebra induced by $S$. The dependency of the hypothesis on the dataset is measured by the relative entropy $\relent( \bP_W^S \Vert \bQ_W)$ of the algorithm's hypothesis kernel $\bP_W^S$, or \emph{posterior}, with respect to %
the data-independent distribution $\bQ_W$, or \emph{prior}, on the hypothesis space. When the confidence penalty is logarithmic, that is, $\log \nicefrac{1}{\beta}$, we say that the bound is of \emph{high probability}.

Note that the PAC-Bayesian guarantee from~\eqref{eq:mc_allester} is on the algorithm's output distribution $\bP_W^S$, and not on any particular realization from it. To simplify the notation, in the rest of the paper we will use $\poprisk \coloneqq \bE^S \poprisk(W)$, $\emprisk \coloneqq \bE^S \emprisk(W,S)$, $\relent \coloneqq \relent(\bP_W^S \Vert \bQ_W)$ and $\comp \coloneqq \relent + \log \nicefrac{\xi(n)}{\beta}$; while understanding that these quantities are random variables whose randomness comes from the random training set $S$.

There have been multiple efforts to generalize McAllester's bound~\eqref{eq:mc_allester} to unbounded losses. These results often require some assumptions on the tail behavior of the random loss $\ell(w,Z)$ with respect to the problem distribution $\bP_Z$ and generalize classical concentration inequalities to the PAC-Bayesian setting. For example, the \emph{cumulative generating function} (CGF) $\Lambda_{\ell(w,Z)}(\lambda) \coloneqq \log \bE \exp \big( \lambda( \ell(w,Z) - \bE \ell(w,Z) \big)$ completely characterizes the tails of $\ell(w,Z)$ for fixed $w$. 
The \emph{Cramér-Chernoff method} determines the connection between the CGF and the tails behavior~\citep[Section 2.3]{boucheron2013concentration}.  
More precisely, if there is a convex and continuously differentiable function $\psi(\lambda)$ defined on $[0,b)$ for some $b \in \bR$ such that $\psi(0)=\psi'(0) = 0$ and
$\Lambda_{-\ell(w,Z)}(\lambda) \leq \psi(\lambda)$ 
for all $\lambda \in [0,b)$, then the \emph{Chernoff inequality} establishes that $\bE \ell(w,Z) \leq \ell(w,Z) + \psi_*^{-1}(\log \nicefrac{1}{\beta})$ with probability no smaller than $1 - \beta$. 
In \citep[Corollary 15]{rodriguezgalvez2023pacbayes}, the authors build on~\citep{zhang2006information, banerjee2021information} to derive a PAC-Bayesian analogue to the Chernoff inequality accounting for the dependence of the training set $S$ and the hypothesis $W$. Namely, with probability no smaller than $1 - \beta$, 
\begin{equation}
    \label{eq:pac_bayes_chernoff_analogue}
    \poprisk \leq \emprisk + \psi_*^{-1} \bigg( \frac{1.1 \relent + \log \frac{10 e \pi^2}{\beta}}{n} \bigg)
\end{equation}
holds \emph{simultaneously} $\forall \ \bP_W^S \in \cP$.
Some examples of losses with a bounded CGF include both \emph{sub-Gaussian} and \emph{sub-exponential} losses, which were also studied individually in~\cite{catoni2004statistical, alquier2006transductive, hellstrom2020generalization, guedj2021still}.

A weaker assumption is to consider losses with bounded moments for all hypotheses $w \in \cW$. 
For a fixed hypothesis $w$, the \emph{$p$-th (raw) moment} of the loss is $\bE \ell(w,Z)^p$. The assumption of bounded moments is weaker since if the CGF exists, then all the moments are bounded. However, the reverse is not true: for example, the log-normal distribution has bounded moments of all orders, but it does not have a CGF~\citep[Chapter 14]{johnson1994continuous}~\citep{asmussen2016laplace}. 
The smaller the order of the bounded moment, the weaker the assumption as $\bE \ell(w,Z)^p \leq \bE \ell(w,Z)^q$ for all $p \leq q$. When the loss has a bounded $p$-th moment but it does not have a CGF, the loss is said to have a \emph{heavy tail}. There are works that obtain PAC-Bayesian bounds similar to~\eqref{eq:pac_bayes_chernoff_analogue} assuming a bounded 2nd moment \citep{wang2015pac, haddouche2023pacbayes, chugg2023unified, alquier2006transductive} or a bounded 2nd and 3rd moments~\citep{holland2019pac}. \citet{alquier2018simpler} also developed PAC-Bayesian bounds for losses with bounded moments, but they considered the \emph{$p$-th central moment} $\bE | \ell(w,Z) - \bE \ell(w,Z) | ^ p$, which can be much smaller than the raw moment. 
However, in these bounds
the confidence penalty $\nicefrac{1}{\beta}$ is linear and not logarithmic, and they consider other $f$-divergences as the dependency measure.

Finally, \citet{haddouche2021pac} considered a different kind of condition called the hypothesis-dependent range (HYPE), which states that there is a function $\kappa$ with positive range such that $\sup_{z \in \cZ} \ell(w,z) \leq \kappa(w)$ for all hypotheses $w \in \cW$;  but their bounds decrease at a slower rate than \eqref{eq:mc_allester}
when they are restricted to the bounded case.

\subsection{Contributions}
\label{subsed:contributions}

In this paper, we build upon \citet{alquier2006transductive}'s truncation method and demonstrate its potential. This method consists of studying a \emph{truncated} version of the loss. 
To this effect, let 
\begin{equation}
    \label{eq:alquier_truncated_loss}
    \ell^-_{\nicefrac{n}{\lambda}}(w,z) \coloneqq \min \{ \ell(w,z), \nicefrac{n}{\lambda} \} 
\end{equation}
and
\begin{equation}
    \label{eq:alquier_tail_loss}
    \ell^+_{\nicefrac{n}{\lambda}}(w,z) \coloneqq \big[ \ell(w,z) - \nicefrac{n}{\lambda} \big]_+
\end{equation}
\looseness=-1 
where $[x]_+ \coloneqq \max \{ x, 0 \}$ and where 
$\lambda \in \bR_+$ is suitably chosen. Thus, we have $\ell(w,z) \leq \ell^-_{\nicefrac{n}{\lambda}}(w,z) + \ell^+_{\nicefrac{n}{\lambda}}(w,z)$. 
Then, one may bound the population risk associated to the truncated loss $\ell^-_{\nicefrac{n}{\lambda}}$ using standard techniques for bounded losses, and translate that to PAC-Bayesian bounds for the unbounded loss $\ell$ accounting for the loss' tail $\bE \ell^+_{\nicefrac{n}{\lambda}}(w,Z)$.

In particular, we focus on losses with heavy tails that have a bounded $p$-th moment. Our contributions are:

\begin{itemize}
    \item  \looseness=-1 We refine the decomposition proposed in \citep{alquier2006transductive} and further study the resulting bounds. In particular, we show that, contrary to what is mentioned in \citep[Section 5.2.1]{alquier2021user}, there are choices of the parameter $\lambda$ such that the term associated to the loss' tail does not dominate and slows down the rate. In fact, we show that the resulting bound's rate is in $\cO \big( n^{- \frac{p-1}{p}} \big)$. This is appealing since it interpolates between a \emph{slow rate} of $\nicefrac{1}{\sqrt{n}}$ when only the 2nd moment is bounded, to a \emph{fast rate} of $\nicefrac{1}{n}$ when all the moments are bounded and the loss is bounded $\bP_Z$-almost surely (a.s.).

    \item For $p = 2$, we derive new high-probability PAC-Bayes bounds for losses with a bounded variance that are tighter than~\citep[Theorem 1]{alquier2018simpler} and~\citep[Corollary 2]{ohnishi2021novel}.

    \item Finally, we extend all the presetned results to bounds in expectation and single-draw PAC-Bayes bounds.
\end{itemize}

\section{Alquier's truncation method}
\label{subsec:alquier-truncation-method}

In his Ph.D. thesis, \citet{alquier2006transductive} discussed a method to find PAC-Bayesian bounds for unbounded losses. This method consists of considering the following bound on the loss
\begin{equation*}
    \ell(w,z) \leq \ell^-_{\nicefrac{n}{\lambda}}(w,z) + \ell^+_{\nicefrac{n}{\lambda}}(w,z),
\end{equation*}
where $\ell^-_{\nicefrac{n}{\lambda}}$ and $\ell^+_{\nicefrac{n}{\lambda}}$ are defined in~\eqref{eq:alquier_truncated_loss} and~\eqref{eq:alquier_tail_loss} respectively. Therefore, the population risk can be bounded as $\poprisk \leq \poprisk^-_{\nicefrac{n}{\lambda}} + \poprisk^+_{\nicefrac{n}{\lambda}}$, where $\poprisk^-_{\nicefrac{n}{\lambda}}$ and $\poprisk^+_{\nicefrac{n}{\lambda}}$ are defined as the population risks of $\ell^-_{\nicefrac{n}{\lambda}}$ and $\ell^+_{\nicefrac{n}{\lambda}}$ respectively. Then, it 
is clear that one can bound each of these two risk terms separately. 

The first term $\poprisk^-_{\nicefrac{n}{\lambda}}$ is especially easy to bound since it has a bounded range in $[0, \nicefrac{n}{\lambda}]$. \citet[Corollary 2.5]{alquier2006transductive} used a bound \emph{à la} \citet{catoni2004statistical}. Instead, we will consider \citep[Theorem 7]{rodriguezgalvez2023pacbayes}, which is as tight as the Seeger--Langford bound~\cite{langford2001bounds, seeger2002pac} and is easier to interpret. To simplify the expressions henceforth, we define $\kappa_1 \coloneqq c \gamma \log \big( \nicefrac{\gamma}{(\gamma -1)} \big)$, $\kappa_2 \coloneqq c \gamma$, and $\kappa_3 \coloneqq \gamma \big( 1 - c(1 - \log c)\big)$, with the understanding that they are functions of the parameters $c \in (0,1]$ and $\gamma > 1$ from \citep[Theorem 7]{rodriguezgalvez2023pacbayes}.

The bound on the second term depends on the tails of the loss and varies depending on the available information. We make this explicit using~\citep[Lemma 4.4]{kallenberg1997foundations}. Namely,
\begin{align*}
    \poprisk^+_{\nicefrac{n}{\lambda}} &= \bE^S \max \Big \{ \ell(W,Z) - \frac{n}{\lambda}, 0 \Big \}  \nonumber \\
    &= \int_0^\infty \bP^S \Big[ \max \Big \{ \ell(W,Z) - \frac{n}{\lambda}, 0 \Big\} > t \Big] \rmd t \nonumber \\
    &\leq \int_0^\infty \bP^S \Big[ \ell(w,Z) > t + \frac{n}{\lambda} \Big] \rmd t \nonumber \\
    \label{eq:bound_to_tail_alquier}
    &= \int_{\frac{n}{\lambda}}^\infty \bP^S \big[ \ell(w,Z) > t \big] \rmd t.
\end{align*}

\begin{lemma}[{\citet[Corollary 2.5, adapted]{alquier2006transductive}}]
     \label{lemma:alquier_truncation_method}
    For all $\beta \in (0,1)$ and all $\lambda > 0$, with probability no smaller than $1 - \beta$,
    \begin{equation*}
    \label{eq:bound_to_alquier}
        \poprisk \leq \kappa_1 \cdot \emprisk^-_{\nicefrac{n}{\lambda}} + \kappa_2 \cdot \frac{\comp}{\lambda} + \kappa_3 \cdot \frac{n}{\lambda}+ \int_{\nicefrac{n}{\lambda}}^\infty \bP^S \big[ \ell(w,Z) > t \big] \rmd t
    \end{equation*}
    holds \emph{simultaneously} $\forall (\bP_W^S, c, \gamma) \in \cP \times (0, 1] \times [1, \infty)$.
\end{lemma}

\looseness=-1 In this way, if we have some knowledge about the tails of the loss, we can trade off (i) the penalty of the loss' tail after a threshold $\nicefrac{n}{\lambda}$ for (ii) the penalty of the range $\nicefrac{n}{\lambda}$ while exploiting the existing sharp bounds for losses with a bounded range.

\subsection{Refining the method}

As hinted later by \citet[Section 5.2.1]{alquier2021user} and made explicit above in~\Cref{lemma:alquier_truncation_method}, this method is rooted into decomposing the loss into a bounded part where $\ell(w,z) \leq \nicefrac{n}{\lambda}$ and an unbounded part where $\ell(w,z) > \nicefrac{n}{\lambda}$. This can be further untangled with the decomposition
\begin{equation*}
\label{eq:alquier_refined_decomposition}
    \ell(w,z) = \ell_{\leq \nicefrac{n}{\lambda}}(w,z) + \ell_{> \nicefrac{n}{\lambda}}(w,z),
\end{equation*}
where 
\begin{align*}
    \ell_{\leq \nicefrac{n}{\lambda}}(w,z) 
    &\coloneqq \ell(w,z) \mathds{1}_{\{\ell(w,z) \leq \nicefrac{n}{\lambda} \}}(w,z), \\
    \ell_{> \nicefrac{n}{\lambda}}(w,z) 
    &\coloneqq \ell(w,z) \mathds{1}_{ \{ \ell(w,z) > \nicefrac{n}{\lambda} \} }(w,z),
\end{align*}
and $\mathds{1}_\cA(w,z)$ is the indicator function returning 1 if $(w,z) \in \cA$ and 0 otherwise.
Therefore, the population risk can be decomposed similarly to before as $\poprisk = \poprisk_{\leq \nicefrac{n}{\lambda}} +  \poprisk_{> \nicefrac{n}{\lambda}}$, where $\poprisk_{\leq \nicefrac{n}{\lambda}}$ and $\poprisk_{> \nicefrac{n}{\lambda}}$ are defined as the population risks of $\ell_{\leq \nicefrac{n}{\lambda}}$ and $\ell_{< \nicefrac{n}{\lambda}}$ respectively. 

Proceeding as before, the two risk terms can be bounded. The first term $\poprisk_{\leq \nicefrac{n}{\lambda}}$ is also bounded in $[0,\nicefrac{n}{\lambda}]$, but it is potentially much smaller than %
$\poprisk^-_{\nicefrac{n}{\lambda}}$
since $\bE^S [ \ell^-_{\nicefrac{n}{\lambda}}(W,Z) | \ell(W,Z) > \nicefrac{n}{\lambda}] = \nicefrac{n}{\lambda}$, while $\bE^S [ \ell_{\leq \nicefrac{n}{\lambda}}(W,Z) | \ell(W,Z) > \nicefrac{n}{\lambda}] = 0$. %
Also, the second term $\poprisk_{> \nicefrac{n}{\lambda}}$ can be bounded by exactly the same quantity as with \citet{alquier2006transductive}'s original decomposition, namely
\begin{align}
    \poprisk_{> \nicefrac{n}{\lambda}} &= \bE^S \ell(W,Z) \mathds{1}_{\ell(w,z) > \nicefrac{n}{\lambda}}(W,S) \nonumber \\
    &= \int_{\nicefrac{n}{\lambda}}^\infty \bP^S \big[ \ell(W,Z) > t \big] \rmd t. 
    \nonumber 
    \label{eq:bound_to_tail_alquier_refined}
\end{align}

\begin{lemma}[{Refinement of \Cref{lemma:alquier_truncation_method}}]
    \label{lemma:alquier_truncation_method_refined}
    For all $\beta \in (0,1)$ and all $\lambda > 0$, with probability no smaller than $1 - \beta$,
    \begin{equation*}
    \label{eq:bound_to_alquier_refined}
       \poprisk \leq \kappa_1 \cdot \emprisk_{\leq \nicefrac{n}{\lambda}} + \kappa_2 \cdot \frac{\comp}{\lambda} + \kappa_3 \cdot \frac{n}{\lambda}+ \int_{\nicefrac{n}{\lambda}}^\infty \bP^S \big[ \ell(w,Z) > t \big] \rmd t
    \end{equation*}
    holds \emph{simultaneously} $\forall (\bP_W^S, c, \gamma) \in \cP \times (0, 1] \times [1, \infty)$.
\end{lemma}

If the tail is bounded by some function $\alpha(n,\lambda)$, i.e., $\int_{\nicefrac{n}{\lambda}}^\infty \bP^S \big[ \ell(w,Z) > t \big] \rmd t \leq \alpha(n,\lambda)$, then the bound resulting from \Cref{lemma:alquier_truncation_method_refined} is optimized by the Gibbs posterior $\rmd \bP_W^{S=s}(w) \propto \rmd \bQ_W(w) e^{- \lambda \cdot \frac{\kappa_1}{\kappa_2} \cdot \emprisk_{\leq \nicefrac{n}{\lambda}}(w,s)}$ independently of $\alpha$.

\section{Losses with a bounded moment}
\label{sec:losses_with_bounded_moment}

\looseness=-1 If the loss has a bounded $p$-th moment $\bE \ell(w,Z)^p \leq m_p < \infty$ for all $w \in \cW$, then one may find PAC-Bayesian bounds using \citet{alquier2006transductive}'s truncation method. More precisely, employing Markov's inequality~\citep[Section 2.1]{boucheron2013concentration} to the term associated to the loss' tail in~\Cref{lemma:alquier_truncation_method_refined} stems the following result.

\begin{lemma}
    \label{lemma:alquier_truncation_method_refined_bouned_moment}
    For every loss with $p$-th moment bounded by $m_p$, for all $\beta \in (0,1)$ and all $\lambda > 0$, with probability no smaller than $1 - \beta$,
    \begin{equation}
    \label{eq:bound_to_alquier_refined_bouned_moment}
       \poprisk \leq \kappa_1 \cdot \emprisk_{\leq \nicefrac{n}{\lambda}} + \kappa_2 \cdot \frac{\comp}{\lambda} + \kappa_3 \cdot \frac{n}{\lambda} + \frac{m_p}{p-1} \Big( \frac{\lambda}{n} \Big)^{p - 1}
    \end{equation}
    holds \emph{simultaneously}  $\forall (\bP_W^S, c, \gamma) \in \cP \times (0, 1] \times [1, \infty)$.
\end{lemma}

\subsection{Alquier's modification for losses with a bounded moment}
\label{subsec:alquier_modification_losses_with_bounded_moment}

\citet[Theorem 2.7]{alquier2006transductive} presented a result similar to \Cref{lemma:alquier_truncation_method_refined_bouned_moment} for losses with a bounded $p$-th moment. However, he did not obtain it with the straightforward technique outlined above. Instead, he considered the truncated loss function 
\begin{equation*}
    \ell_{p,\nicefrac{n}{\lambda}}(w,z) = \bigg[ \ell(w,z) - \frac{1}{p} \Big( \frac{p - 1}{p} \Big)^{p-1} \Big( \frac{\lambda}{n} \Big)^{p - 1} \cdot |\ell(w,z)|^p \bigg]_+.
\end{equation*}

Importantly, this loss function satisfies that $\ell_{p,\nicefrac{n}{\lambda}} \leq \nicefrac{n}{\lambda}$. Then, let $\poprisk_{p,\nicefrac{n}{\lambda}}$ be the population risk associated to $\ell_{p,\nicefrac{n}{\lambda}}$. It directly follows that
\begin{equation*}
    \poprisk \leq \poprisk_{p,\nicefrac{n}{\lambda}} + \frac{1}{p} \Big(\frac{p-1}{p} \Big)^{p-1}  \Big( \frac{\lambda}{n} \Big)^{p - 1} \cdot \bE^S|\ell(W,Z)|^p.
\end{equation*}

\looseness=-1 In this way, like before, the term $\poprisk_{p,\nicefrac{n}{\lambda}}$ can be bounded using any standard PAC-Bayes bound for bounded losses and now the second term is bounded by construction. As before, we will present the result using \citep[Theorem 7]{rodriguezgalvez2023pacbayes} instead of a bound \emph{à la} \citet{catoni2004statistical}. For this purpose, let $\emprisk_{p, \nicefrac{n}{\lambda}}$ be the empirical risk associated to the loss $\ell_{p, \nicefrac{n}{\lambda}}$.

\begin{lemma}[{\citet[Theorem 2.7, adapted]{alquier2006transductive}}]
    \label{lemma:alquier_truncation_method_modified_bounded_moment}
    For every loss with a $p$-th moment bounded by $m_p$, for all $\beta \in (0,1)$ and all $\lambda > 0$, with probability no smaller than $1 - \beta$,
    \begin{equation*}
    \label{eq:original_bound_to_alquier}
       \poprisk \leq \kappa_1 \cdot \emprisk_{p, \nicefrac{n}{\lambda}} + \kappa_2 \cdot \frac{\comp}{\lambda} + \kappa_3 \cdot \frac{n}{\lambda}+ \frac{m_p}{p} \Big(\frac{p-1}{p} \Big)^{p-1}  \Big( \frac{\lambda}{n} \Big)^{p - 1}
    \end{equation*}
    holds \emph{simultaneously} $\forall (\bP_W^S, c, \gamma) \in \cP \times (0, 1] \times [1, \infty)$.
\end{lemma}

Comparing \Cref{lemma:alquier_truncation_method_modified_bounded_moment} to the truncation method with the straightforward \Cref{lemma:alquier_truncation_method_refined_bouned_moment}, we see that the result stemming from \citet{alquier2006transductive}'s modified construction improves the constant of the term associated to the tail from $\nicefrac{1}{p-1}$ to $(\nicefrac{p-1}{p})^{p-1} \cdot \nicefrac{1}{p}$. For $p=2$, the constant is $4$ times smaller changing from $1$ to $\nicefrac{1}{4}$; and for $p \to \infty$ the constant is $e$ times smaller, although both tend to 0. On the other hand, $\emprisk_{\leq \nicefrac{n}{\lambda}}$ has the potential to be smaller than $\emprisk_{p, \nicefrac{n}{\lambda}}$. The results derived in the rest of the paper use \Cref{lemma:alquier_truncation_method_refined_bouned_moment} as a starting point, but analogous results trivially follow from \Cref{lemma:alquier_truncation_method_modified_bounded_moment} with slightly different constants and changing $\emprisk_{\leq \nicefrac{n}{\lambda}}$ to $\emprisk_{p, \nicefrac{n}{\lambda}}$.

In \Cref{lemma:alquier_truncation_method,lemma:alquier_truncation_method_modified_bounded_moment,lemma:alquier_truncation_method_refined,lemma:alquier_truncation_method_refined_bouned_moment}, the term $\nicefrac{\kappa_3 n}{\lambda}$ does not affect the bound's rate as choosing $c = 1$ implies $\kappa_3 = 0$. The coefficients $\kappa_1$ and $\kappa_2$ are chosen adaptively to minimize the empirical risk and complexity contributions as discussed in~\citep{rodriguezgalvez2023pacbayes}.

\subsection{Optimizing the parameter in the bound}
\label{subsec:optimizing_parameter_p_greater_2}

\looseness=-1 \citet{alquier2006transductive, alquier2021user} considered the \emph{data-independent} $\lambda = \sqrt{n}$. This gives a bound with a rate of $\nicefrac{1}{\sqrt{n}}$ for any loss with a bounded $p$-th moment where $p > 2$. A better choice is $\lambda = \big( \nicefrac{n^{p-1}}{m_p} \big)^{\nicefrac{1}{p}}$. This results in a bound with a rate of $n^{- \frac{p - 1}{p}}$.

\begin{theorem}
    \label{th:alquier_truncation_method_refined_fixed_lambda}
    For every loss with a bounded $p$-th moment, for all $\beta \in (0,1)$, with probability no smaller than $1 - \beta$,
    \begin{equation*}
       \poprisk \leq \kappa_1 \cdot \emprisk_{\leq (m_p n)^{\frac{1}{p}}} + \Big( \frac{m_p}{n^{p-1}} \Big)^{\frac{1}{p}} \Big( \kappa_2 \cdot \comp + \kappa_3 \cdot n + \frac{1}{p-1} \Big)
    \end{equation*}
    holds \emph{simultaneously} $\forall (\bP_W^S, c, \gamma) \in \cP \times (0, 1] \times [1, \infty)$.
\end{theorem}

In this way, the rate for $p = 2$ is exactly the same, a slow rate of $\nicefrac{1}{\sqrt{n}}$. However, as the order of the known bounded moment increases, that is $p \to \infty$, the rate becomes a fast rate of $\nicefrac{1}{n}$. Hence, this choice of $\lambda$ allows us to interpolate between a slow and a fast rate depending on how much knowledge about the tails is available to us. Furthermore, as we gain knowledge of the tails, the truncation of the loss $\ell_{\leq (m_p n)^{\nicefrac{1}{p}}}$ becomes less dependent on the number of training data $n$ and in the limit $p \to \infty$ only depends of the $\bP_Z$-a.s. boundedness of the loss, namely $\lim_{p \to \infty} (m_p n)^{\nicefrac{1}{p}} = \sup_{w \in \cW} \esssup \ell(w,Z)$.

\looseness=-1 Instead of choosing a data-independent parameter $\lambda$, we can use the event space discretization technique from \citep{rodriguezgalvez2023pacbayes} to get a better dependence on the relative entropy. In particular, the following result follows by not considering any ``uninteresting event'' and following the technique as outlined in \citep[Corollary 15]{rodriguezgalvez2023pacbayes}. Henceforth, let us define $\comp' \coloneqq 1.1 \relent + \log \nicefrac{10 e \pi^2 \xi(n)}{\beta}$. The full proof is given in~\Cref{app:proof_alquier_truncation_method_refined_adaptive_lambda}.

\begin{theorem}
    \label{th:alquier_truncation_method_refined_adaptive_lambda}
    For every loss with a $p$-th moment bounded by $m_p$, for all $\beta \in (0,1)$, with probability no smaller than $1 - \beta$,
    \begin{equation*}
        \poprisk \leq \kappa_1 \cdot \emprisk_{\leq t^\star} + m_p^{\frac{1}{p}} \Big(\frac{p}{p-1}\Big) \Big( \kappa_2 \cdot \frac{\comp'}{n} + \kappa_3 \Big)^{\frac{p-1}{p}}
    \end{equation*}
    holds \emph{simultaneously} $\forall (\bP_W^S, c, \gamma) \in \cP \times (0, 1] \times [1, \infty)$,
    where $$t^\star \coloneqq m_p^{\frac{1}{p}} \Big( \kappa_2 \cdot \frac{\comp'}{n} + \kappa_3 \Big)^{-\frac{1}{p}}.$$
\end{theorem}

In this way, the rate is maintained, while the dependence on the relative entropy changed from linear to polynomial of order $\nicefrac{(p-1)}{p}$. For order $p = 2$, this corresponds to the square root, and only goes to the linear case when $p \to \infty$, when we also achieve a fast rate of $\nicefrac{1}{n}$.

Following the insights of \citep[Section 3.2.4]{rodriguezgalvez2023pacbayes}, we may use~\Cref{th:alquier_truncation_method_refined_adaptive_lambda} to obtain an equivalent result, but in the form of~\Cref{lemma:alquier_truncation_method_refined} that holds \emph{simultaneously} for all $\lambda$.

\begin{theorem}
    \label{th:alquier_truncation_method_refined_simultaneously_all_lambda}
    For every loss with a $p$-th moment bounded by $m_p$, for all $\beta \in (0,1)$, with probability no smaller than $1 - \beta$, 
    \begin{equation*}
    \label{eq:bound_to_alquier_refined_simultaneously_all_lambda}
       \poprisk \leq \kappa_1 \cdot \emprisk_{\leq \nicefrac{n}{\lambda}} + \kappa_2 \cdot \frac{\comp'}{\lambda} + \kappa_3 \cdot \frac{n}{\lambda} +\frac{m_p}{p-1} \cdot \Big( \frac{\lambda}{n} \Big)^{p - 1}
    \end{equation*}
    holds \emph{simultaneously} $\smash{\forall ( \bP_W^S, c, \gamma, \lambda) \in \cP\! \times \! (0, 1] \! \times \! [1, \infty) \! \times \! \bR_+}$.
\end{theorem}

From \Cref{th:alquier_truncation_method_refined_simultaneously_all_lambda}, we understand that the posterior that optimizes both \Cref{th:alquier_truncation_method_refined_adaptive_lambda,th:alquier_truncation_method_refined_simultaneously_all_lambda} is the Gibbs posterior $\rmd \bP_W^{S=s}(w) \propto \rmd \bQ_W(w) e^{- \frac{\lambda}{2} \cdot \frac{\kappa_1}{\kappa_2} \emprisk_{\leq \nicefrac{n}{\lambda}}(w,s)}$, where now $c$, $\lambda$, and $\gamma$ can be chosen \emph{adaptively after} observing the realization of the data $s$. This way, the choice of the parameter $\lambda$ can be made to optimize the bound emerging from that data realization. On the other hand, the Gibbs distribution emerging from the optimization of \Cref{lemma:alquier_truncation_method_refined} needs to commit to a \emph{fixed} parameter $\lambda$ \emph{before} observing the training data and is data-independent.

\subsection{The case \texorpdfstring{$p \to \infty$}{} and essentially bounded losses}
\label{subsec:case_p_infinity}

So far we only considered the algorithm-independent condition of losses with a bounded $p$-th moment $\bE \ell(w,Z)^p$ for all $w \in \cW$. This condition only depends on the loss and the problem distribution $\bP_Z$. Nonetheless, all the previous results can be replicated under the weaker condition that the loss has a bounded $p$-th moment with respect to the algorithm's output, that is, that $m_p' \coloneqq \bE^S \ell(W,Z)^p$ is bounded $\bP_S$-a.s.

Although this condition is weaker, it is harder to guarantee as it requires some knowledge of the data distribution $\bP_Z$ \emph{and} the algorithm's Markov kernel $\bP_W^S$. This knowledge could instead be used to directly find a bound on $\poprisk = \bE^S \ell(W,Z)$. 

\looseness=-1 However, results under this condition can be useful in some situations. For example, they can be used to derive new results for losses with a bounded variance (as shown later in~\Cref{sec:bounds_p_equal_2}) and to obtain more meaningful findings when $p \to \infty$.

\Cref{th:alquier_truncation_method_refined_adaptive_lambda}, when specialized to $p \to \infty$, gives us a fast rate result when the loss is $\bP_Z$-a.s. bounded, that is, when $\esssup \ell(w,Z) < \infty$ for all $w \in \cW$. 
This condition of the loss being $\bP_Z$-\emph{essentially bounded} can be a strong requirement, similar to the one of bounded losses.
However, when we have more information about the algorithm, then %
we can obtain a fast rate result when the loss is $\bP_{W,S} \otimes \bP_Z$-a.s. bounded, that is, when $\esssup \ell(W,Z) < v$. 
This condition is much weaker than the previous essential boundedness or just boundedness of the loss. Namely, one needs to know that the algorithm is such that $\bP ( \ell(W,Z) < v ) = 1$. As an example, consider the squared loss $\ell(w,z) = (w - z)^2$ and some data that belongs to some interval of length 1 with probability 1, that is $\bP ( Z \in [a, a+1])=1$, but where we ignore $a$. Consider $w \in \bR$, the simple algorithm that returns the average of the training instances $\bA(s) = \sum_{i=1}^n \nicefrac{z_i}{n}$ ensures that $\esssup \ell(W,Z) < 1$, while $\sup_{w \in \bR} \esssup \ell(w,Z) \to \infty$.

\section{Losses with a bounded variance}
\label{sec:bounds_p_equal_2}

\begin{figure*}[t]
    \centering
    \includegraphics[width=0.24\textwidth]{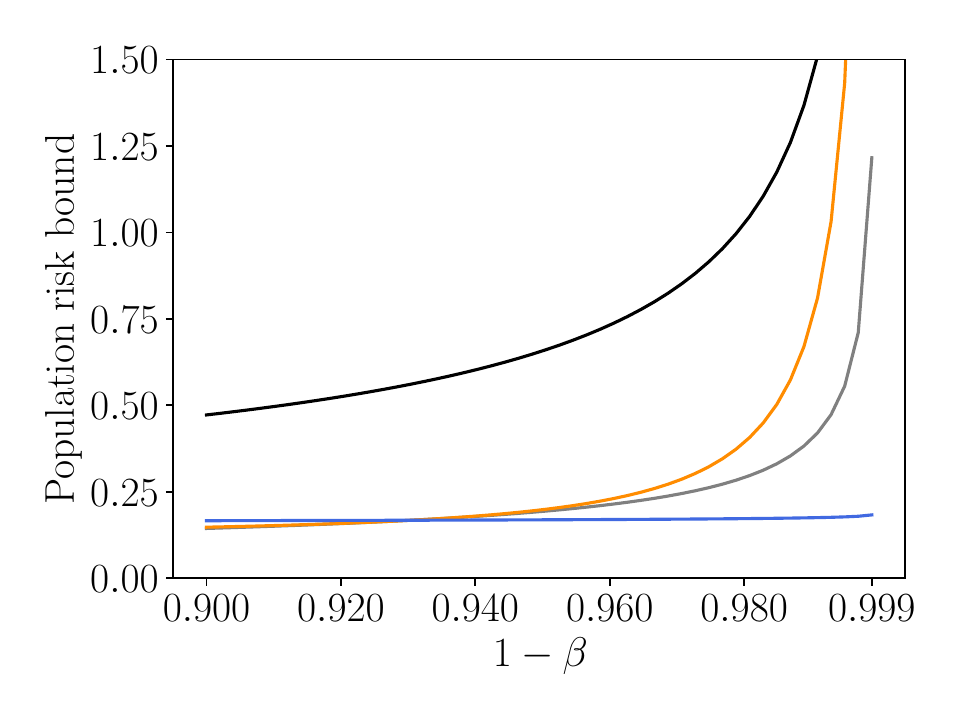}
    \includegraphics[width=0.24\textwidth]{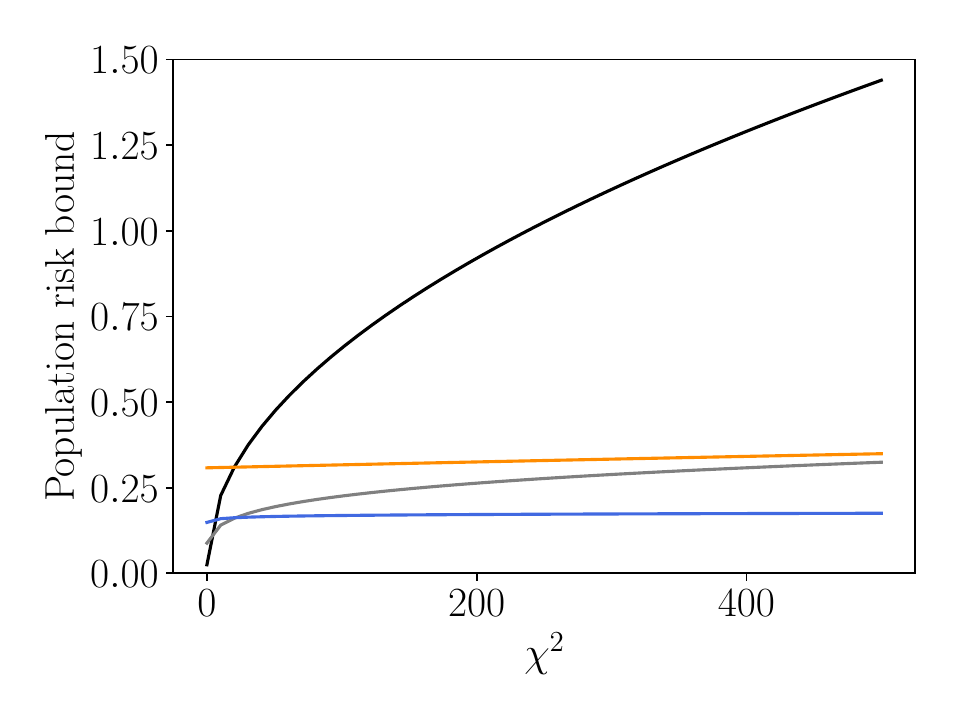}
    \includegraphics[width=0.24\textwidth]{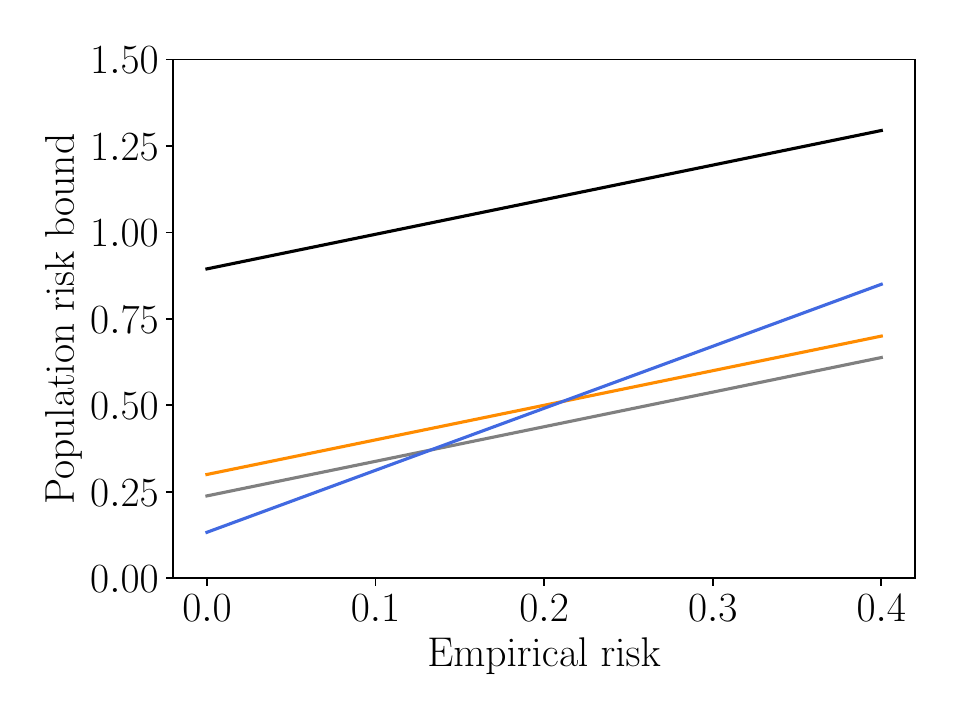}
    \includegraphics[width=0.24\textwidth]{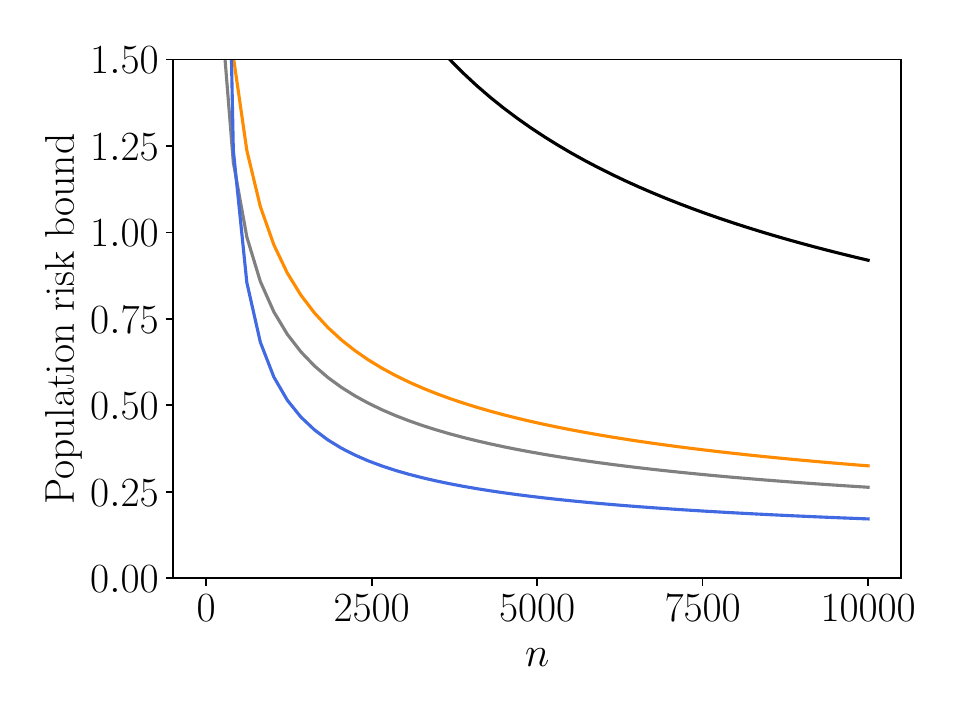}
    \caption{Illustration comparing \citep{alquier2018simpler, ohnishi2021novel} (\eqref{eq:bounded_variance_alquier} in black, \eqref{eq:bounded_variance_honorio_1} in gray, and \eqref{eq:bounded_variance_honorio_2} in orange) and our \Cref{th:bounded_variance_high_probability} (in blue) for varying values of $\beta$, $\chi^2$, $\emprisk$, and $n$. To help the comparison, we actually use the upper bound relaxation~\eqref{eq:relaxation_of_bounded_variance_high_probability} of \Cref{th:bounded_variance_high_probability}. When they are not varying, the values of the parameters are fixed to $\beta = 0.025$, $\chi^2 = 200$, $\emprisk = 0.025$, $n=10,000$, and $\sigma^2 = 1$.}
    \label{fig:comparison_bounds_bounded_variance}
\end{figure*}

\looseness=-1 A particularly important case is the one of losses with a bounded second moment. \Cref{th:alquier_truncation_method_refined_adaptive_lambda} recovers the expected rate of $\sqrt{\nicefrac{m_2 \relent}{n}}$ from~\citep[Theorem 11]{rodriguezgalvez2023pacbayes}. This is the smallest moment with a rate no slower than $\nicefrac{1}{\sqrt{n}}$. However, as mentioned in~\citep{rodriguezgalvez2023pacbayes}, the raw second moment $m_2$ can be much larger than the \emph{variance}, or central second moment. When the variance is bounded, that is $\bE (\ell(w,Z) - \bE \ell(w,Z))^2 \leq \sigma^2 < \infty$ for all $w \in \cW$, the only PAC-Bayesian results we are aware of are \cite{alquier2018simpler, ohnishi2021novel}.

\begin{theorem}[{\citet[Theorem 1]{alquier2018simpler} and \citet[Corollary 2]{ohnishi2021novel}}]
    \label{th:bounded_variance_not_high_probability}
    \looseness=-1 For every loss with a variance bounded by $\sigma^2$, for all $\beta \in (0,1)$, with probability no smaller than $1 - \beta$, %
    each of the inequalities
    \begin{align}
        \label{eq:bounded_variance_alquier}
        \poprisk &\leq \emprisk + \sqrt{\frac{\sigma^2 (\chi^2 + 1)}{n \beta}} %
        \\
        \label{eq:bounded_variance_honorio_1}
        \poprisk &\leq \emprisk + \sqrt{\frac{\sigma^2 \sqrt{\chi^2 + 1}}{n \beta}} %
        \\
        \label{eq:bounded_variance_honorio_2}
        \poprisk &\leq \emprisk + \sqrt{\frac{\chi^2 + \big(\nicefrac{\sigma^2}{\beta}\big)^2}{2n}} %
    \end{align}
    hold \emph{simultaneously} $\forall \ \bP_W^S \in \cP$,
    where $\chi^2 \coloneqq \chi^2 (\bP_W^S, \bQ_W)$ is the chi-squared divergence.%
    \footnote{The result is originally given by $\mathrm{Var}^S(\ell(W,Z))$, which usually requires too much knowledge on the algorithm and data distributions. We presented it with the algorithm-independent variance $\sigma^2 \geq \mathrm{Var}^S(\ell(W,Z))$.}
\end{theorem}

Although this bound still achieves an expected slow rate of $\nicefrac{1}{\sqrt{n}}$, there are two main differences between this theorem and those presented in the preceding sections. First, and most notable, the dependence with the confidence penalty $\nicefrac{1}{\beta}$ is not logarithmic, but polynomial. This can result in a loose bound when high confidence is demanded: for example, for $\beta = 0.05$ we have that $\log \nicefrac{1}{\beta} \approx 3$ while $\nicefrac{1}{\beta} = 20$. Second, the dependency measure changed from the relative entropy $\relent$ to the chi-squared divergence $\chi^2$. The chi-squared divergence also measures the dissimilarity between the posterior $\bP_W^S$ and the prior $\bQ_W$, but it can be much larger. More precisely,
\begin{equation}
    \label{eq:relent_chi_squared_bound}
    0 \leq \relent \leq \log(1 + \chi^2) \leq \chi^2
\end{equation}
and no lower bound of the relative entropy $\relent$ is possible in terms of the chi-squared divergence $\chi^2$~\citep[Section 7.7]{polyanskiy2022lecture}. %

Studying \Cref{th:alquier_truncation_method_refined_adaptive_lambda} with the weaker condition that $\bE \ell(W,Z)^2 \leq m_2'$ as discussed in~\Cref{subsec:case_p_infinity}, we can obtain a high-probability PAC-Bayes bound for losses with a bounded variance that has the relative entropy as the dependency measure. As in the previous section, the method and proof technique also extends to an analysis starting from \Cref{lemma:alquier_truncation_method_modified_bounded_moment} resulting in slightly different constants and using $\emprisk_{p, \nicefrac{n}{\lambda}}$ as an estimator instead of $\emprisk_{\leq \nicefrac{n}{\lambda}}$. Similarly, the method also extends to the semi-empirical bound from~\citep[Theorem 11]{rodriguezgalvez2023pacbayes}.

\begin{theorem}
    \label{th:bounded_variance_high_probability}
    For every loss with a variance bounded by $\sigma^2$, for all $\beta \in (0,1)$, with probability no smaller than $1 - \beta$,
    \begin{equation*}
        \poprisk \leq \Big[ 1 - 2 \sqrt{\comp''} \Big]_+^{-1} \Big[ \kappa_1 \cdot \emprisk  + 2 \sqrt{\sigma^2 \comp''} \Big]
    \end{equation*}
    holds \emph{simultaneously} $\forall (\bP_W^S, c, \gamma) \in \cP \times (0, 1] \times [1, \infty)$,
    where $\comp'' \coloneqq \nicefrac{\kappa_2 \comp'}{n} + \kappa_3$.
\end{theorem}

\begin{proof}[Sketch of the proof]
    Consider the relaxed version of \Cref{th:alquier_truncation_method_refined_adaptive_lambda} from \Cref{subsec:case_p_infinity} for $p=2$ and note that $m_2' = \textnormal{Var}^S(\ell(W,Z)) + \poprisk^2$. Then, for all $\beta \in (0,1)$, with probability no smaller than $1 - \beta$ we have
    \begin{equation*}
        \poprisk \leq \kappa_1 \cdot \emprisk + 2 \sqrt{\big(\textnormal{Var}^S(\ell(W,Z)) + \poprisk^2 \big) \cdot \comp''}
    \end{equation*}
    simultaneously for all $c \in (0,1]$ and all $\gamma > 1$, where $\comp''$ is defined as in the theorem statement. Then, we may employ the inequality $\sqrt{x + y} \leq \sqrt{x} + \sqrt{y}$ to separate the square root and the inequality $\mathrm{Var}^S(\ell(W,Z)) \leq \sup_{w \in \cW} \mathrm{Var}(\ell(w,Z)) = \sigma^2$ to obtain our algorithm-independent variance. After that, re-arranging the equation and accepting the convention that $1/0 \to \infty$ completes the proof. The full proof is in~\Cref{app:proof_bounded_variance_high_probability}.
\end{proof}

Although the \Cref{th:bounded_variance_high_probability} is of high probability and considers the relative entropy, it is hard to compare \Cref{th:bounded_variance_not_high_probability} due to the first factor $[ 1 -  2 (\comp'')^{\nicefrac{1}{2}}]_+^{-1}$. This factor ensures the bound is only useful when $ 2 (\comp'')^{\nicefrac{1}{2}} < 1$, which is the range where the bound would be effective without the said factor anyway. To effectively compare the two bounds, we bound \Cref{th:bounded_variance_high_probability} from above using the relative entropy upper bound~\eqref{eq:relent_chi_squared_bound}, that is,
\begin{equation}
\label{eq:relaxation_of_bounded_variance_high_probability}
     \poprisk \leq \Big[ 1 - 2 \sqrt{\mathfrak{C}_{n,\beta,S,\chi^2}''} \Big]_+^{-1} \Big[ \kappa_1 \cdot \emprisk  + 2 \sqrt{\sigma^2 \mathfrak{C}_{n,\beta,S,\chi^2}''} \Big]
\end{equation}
where 
\begin{equation*}
    \mathfrak{C}_{n,\beta,S,\chi^2}'' \coloneqq \kappa_2 \cdot \frac{1.1 \log (1 + \chi^2) + \log \frac{10 e \pi^2 \xi(n)}{\beta}}{n} + \kappa_3.
\end{equation*}
Also, we fix $c = 1$ and $\gamma = \nicefrac{e}{(e - 1)}$. 
Even with this relaxation, the presented high probability bound is tighter than \Cref{th:bounded_variance_not_high_probability} in many regimes (see \Cref{fig:comparison_bounds_bounded_variance}).

\section{Extension of the results}
\label{subsec:extension_of_the_bounds}

Although \citet{alquier2006transductive} devised the truncation method for PAC-Bayes bounds and we presented our results in this setting, there is nothing stopping us to use this technique to derive bounds in expectation or single-draw PAC-Bayes bounds.

Bounds in expectation and single-draw PAC-Bayes bounds are similar to the PAC-Bayes bounds from~\Cref{subsec:pac-bayesian-bounds} and~\eqref{eq:mc_allester}. Bounds in expectation provide weaker generalization guarantees. Here, the \emph{expected} population risk $\bE  \poprisk(W) $ is bounded using the \emph{expected} empirical risk $\bE \emprisk(W,S)$. That is, the bound holds \emph{on average} over the draw of the random training set $S$ and the returned hypothesis $W$, and there is no confidence parameter. Single-draw PAC-Bayes bounds, on the other hand, provide stronger generalization guarantees. More precisely, they provide guarantees for the population risk $\poprisk(W)$ that hold with probability $1 - \beta$ with respect to the draw of the random training set $S$ and the returned hypothesis $W$.

\looseness=-1 In \Cref{app:extension_bounds_in_expectation,app:extension_single_draw_pac_bayes_bounds}, we derive ``in expectation'' and ``single-draw PAC-Bayes'' analogues to the PAC-Bayes fast rate bound from~\citep{rodriguezgalvez2023pacbayes}. Then, all the presented results extend to those settings routinely, and they are collected in the appendix for completeness.

\bibliographystyle{IEEEtranN}
{\footnotesize \balance \bibliography{references}}

\newpage
\appendix

\section{Omitted proofs}

\subsection{Proof of \texorpdfstring{\Cref{th:alquier_truncation_method_refined_adaptive_lambda}}{Theorem 2}}
\label{app:proof_alquier_truncation_method_refined_adaptive_lambda}

\looseness=-1 Consider~\eqref{eq:bound_to_alquier_refined_bouned_moment} from~\Cref{lemma:alquier_truncation_method_refined_bouned_moment} and note that the only random element is $\comp$. Let $\cB_k$ be the complement of the event in~\eqref{eq:bound_to_alquier_refined_bouned_moment} with parameters $\beta_k \in (0,1)$ and $\lambda_k > 0$ such that $\bP[\cB_k] < \beta_k$.
Then, further let $\beta_k = \nicefrac{6}{\pi^2} \cdot \nicefrac{\beta}{k^2}$ and define the sub-events $\cE_k \coloneqq \{ k-1 \leq \relent < k \}$ and the indices $\cK \coloneqq \{s \in \cZ^n : k \in \bN : \bP[\cE_k] > 0 \}$. In this way, for all $\beta \in (0,1)$ and all $\lambda_k > 0$, with probability no larger than $\bP[\cB_k | \cE_k]$, there exists a posterior $\bP_W^S \in \cP$, a $c \in (0,1]$, and a $\gamma > 1$ such that
\begin{equation*}
    \poprisk > \kappa_1 \cdot \emprisk_{\leq \frac{n}{\lambda_k}} + \kappa_2 \cdot \tfrac{k + \log \frac{\pi^2 \xi(n) k^2}{6 \beta}}{\lambda_k} + \kappa_3 \cdot \tfrac{n}{\lambda_k} + \frac{m_p}{p-1} \Big( \frac{\lambda_k}{n} \Big)^{p - 1}.
\end{equation*}

Optimizing the parameter $\lambda_k$ guarantees that for all $\beta \in (0,1)$ and all $\lambda_k > 0$, with probability no larger than $\bP[\cB_k | \cE_k]$, there exists a posterior $\bP_W^S \in \cP$, a $c \in (0,1]$, and a $\gamma > 1$ such that
\begin{equation*}
    \poprisk > \kappa_1 \cdot \emprisk_{\leq t_k^{\star}} + m_p^{\frac{1}{p}} \Big(\frac{p}{p-1}\Big) \Big( \kappa_2 \cdot \frac{k + \log \frac{\pi^2 \xi(n) k^2}{6 \beta}}{n} + \kappa_3 \Big)^{\frac{p-1}{p}},
\end{equation*}
where
\begin{equation*}
    t_k^\star \coloneqq m_p^{\frac{1}{p}} \Big( \kappa_2 \cdot \frac{k + \log \frac{\pi^2 \xi(n) k^2}{6 \beta}}{n} + \kappa_3 \Big)^{-\frac{1}{p}}.
\end{equation*}

Then, noting that $k \leq \relent + 1$ given $\cE_k$, noting that the inequality $x + \log \frac{e \pi^2 (x+1)^2}{6 \beta} \leq \big(\frac{a+3}{a+1}\big) x + \log \frac{e\pi^2 (a+1)^2}{6 \beta} - \frac{2a}{a+1}$ holds for all $a >0$, and using this inequality with $a = 19$, we have that for all $\beta \in (0,1)$ and all $\lambda_k > 0$, with probability no larger than $\bP[\cB_k | \cE_k]$, there exists a posterior $\bP_W^S \in \cP$, a $c \in (0,1]$, and a $\gamma > 1$ such that
\begin{equation}
\label{eq:bounded_moment_lambda_optimization}
    \poprisk > \kappa_1 \cdot \emprisk_{\leq t^\star} + m_p^{\frac{1}{p}} \Big(\frac{p}{p-1}\Big) \Big( \kappa_2 \cdot \frac{\comp'}{n} + \kappa_3 \Big)^{\frac{p-1}{p}},
\end{equation}
where $$t^\star \coloneqq m_p^{\frac{1}{p}} \Big( \kappa_2 \cdot \frac{\comp'}{n} + \kappa_3 \Big)^{-\frac{1}{p}}.$$

If we let $\cB$ be the event described by~\eqref{eq:bounded_moment_lambda_optimization}, we can bound its probability by
\begin{align*}
    \bP[\cB] &= \sum_{k \in \cK} \bP[\cB | \cE_k] \bP[\cE_k] \leq \sum_{k \in \cK} \bP[\cB_k | \cE_k] \bP[\cE_k] \leq \sum_{k \in \cK} \bP[ \cB_k]
\end{align*}
and therefore $\bP[\cB] < \beta$, which completes the proof. \hfill $\qedsymbol$

\subsection{Proof of \texorpdfstring{\Cref{th:bounded_variance_high_probability}}{Theorem 5}}
\label{app:proof_bounded_variance_high_probability}

Consider the relaxed version of \Cref{th:alquier_truncation_method_refined_adaptive_lambda} from \Cref{subsec:case_p_infinity} for $p=2$ and note that $m_2' = \textnormal{Var}^S(\ell(W,Z)) + \poprisk^2$. Then, for all $\beta \in (0,1)$, with probability no smaller than $1 - \beta$,
\begin{equation*}
    \poprisk \leq \kappa_1 \cdot \emprisk + 2 \sqrt{\big(\textnormal{Var}^S(\ell(W,Z)) + \poprisk^2 \big) \cdot \comp''}
\end{equation*}
holds simultaneously for all posteriors $\bP_W^S$, all $c \in (0,1]$, and all $\gamma > 1$, where $\comp''$ is defined as in the theorem statement. 

Then, we may employ the inequality $\sqrt{x + y} \leq \sqrt{x} + \sqrt{y}$ to separate the square root and the inequality $$\mathrm{Var}^S(\ell(W,Z)) \leq \sup_{w \in \cW} \mathrm{Var}(\ell(w,Z)) = \sigma^2$$ to obtain our algorithm-independent variance. In this way, for all $\beta \in (0,1)$, with probability no smaller than $1 - \beta$,
\begin{equation*}
    \poprisk \leq \kappa_1 \cdot \emprisk + 2 \sqrt{\sigma^2 \cdot \comp''} + 2 \poprisk \cdot \sqrt{\comp''}
\end{equation*}
holds simultaneously for all posteriors $\bP_W^S$, all $c \in (0,1]$, and all $\gamma > 1$.

Re-arranging the equation proves the theorem statement: when $1 \geq 2 (\comp'')^{\nicefrac{1}{2}}$, the theorem holds by the reasoning above, and when $1 \leq 2 (\comp'')^{\nicefrac{1}{2}}$, the theorem holds trivially by the convention that $\nicefrac{1}{0} \to \infty$.

\subsection{Extension to bounds in expectation}
\label{app:extension_bounds_in_expectation}

To start, we first obtain an ``in expectation'' analogue to the PAC-Bayes fast rate bound from~\cite{rodriguezgalvez2023pacbayes}.

\begin{theorem}
    \label{th:fast_rate_mi}
    For every loss with a range bounded in $[0,b]$, the inequality
    \begin{equation*}
        \bE[\poprisk(W)] \leq \kappa_1 \cdot \bE[ \emprisk(W,S)] + b \left[ \kappa_2 \cdot \frac{\minf(W;S)}{n} + \kappa_3 \right]
    \end{equation*}
    holds for all $c \in (0,1]$ and all $\gamma > 1$, where $\kappa_1 \coloneqq c \gamma \log \big( \nicefrac{\gamma}{(\gamma -1)} \big)$, $\kappa_2 \coloneqq c \gamma$, and $\kappa_3 \coloneqq \gamma \big( 1 - c(1 - \log c)\big)$.
\end{theorem}

\begin{proof}
    The proof starts by recalling \citep[Theorem 1.2.6]{catoni2007pac}. This states that for every loss with a range bounded in $[0,1]$,
    \begin{equation*}
        \bE[\poprisk(W)] \leq \frac{1}{1-e^{-\frac{\lambda}{n}}} \left[1 - e^{- \frac{\lambda}{n} \cdot \bE[\emprisk(W,S)] - \frac{\minf(W;S)}{n}} \right]
    \end{equation*}
    holds for all $\lambda > 0$. First, we can do the change of variable $\lambda \coloneqq n \gamma \log \big(\nicefrac{\gamma}{\gamma - 1} \big)$ such that $\gamma > 1$. After that, we can use that the function $1 - e^{-x}$ is a non-decreasing, concave, continuous function for $x > 0$ and therefore can be upper-bounded by its envelope, that is, $1 - e^{-x} = \inf_{a > 0} \{ e^{-a} x + 1 - e^{-a} (1 + a) \}$. Using the envelope in the equation above and letting $c \coloneqq e^{-a} \in (0,1]$  completes the proof for losses with a range bounded in $[0,1]$.
    Finally, the proof is completed by scaling the loss appropriately.
\end{proof}

A single-letter version of~\Cref{th:fast_rate_mi} can be easily derived if we consider an estimation of the population risk with a single sample $\ell(W,Z_i)$. In this way, \Cref{th:fast_rate_mi} states that for every loss with a range bounded in $[0,b]$, the inequality
\begin{equation*}
    \bE[\poprisk(W)] \leq \kappa_{1,i} \cdot \bE[ \ell(W,Z_i) ] + b \left[ \kappa_{2,i} \cdot \minf(W;Z_i) + \kappa_{3,i} \right]
\end{equation*}
holds for all $c_i \in (0,1]$ and all $\gamma_i > 1$, where $\kappa_{1,i} \coloneqq c_i \gamma_i \log \big( \nicefrac{\gamma_i}{(\gamma_i -1)} \big)$, $\kappa_{2,i} \coloneqq c_i \gamma_i$, and $\kappa_{3,i} \coloneqq \gamma_i \big( 1 - c_i(1 - \log c_i)\big)$. Then, taking the average of the theorem for all instances $Z_i$ yields the following result.

\begin{theorem}
    \label{th:fast_rate_mi_single_letter}
    For every loss with a range bounded in $[0,b]$,
    \begin{equation*}
        \bE[\poprisk(W)] \leq \Bar{\kappa}_1 \cdot \bE[ \emprisk(W,S)] + b \left[ \Bar{\kappa}_2 \cdot \frac{1}{n} \sum_{i=1}^n \minf(W;Z_i) + \Bar{\kappa_3} \right]
    \end{equation*}
    holds for all $c_i \in (0,1]$ and all $\gamma_i > 1$, where $\kappa_{1,i} \coloneqq c_i \gamma_i \log \big( \nicefrac{\gamma_i}{(\gamma_i -1)} \big)$, $\kappa_{2,i} \coloneqq c_i \gamma_i$, $\kappa_{3,i} \coloneqq \gamma_i \big( 1 - c_i(1 - \log c_i)\big)$, and $\Bar{\kappa}_j \coloneqq \sum_{i=1}^n \nicefrac{\kappa_{j,i}}{n}$ for all $j \in \{1, 2, 3 \}$.
\end{theorem}

This single-letter theorem is tighter than~\Cref{th:fast_rate_mi} since $\sum_{i=1}^n \minf(W;Z_i) \leq \minf(W;S)$ and since one could chose $\kappa_{i,j} = \kappa_{j}$ for all $j \in \{1, 2, 3 \}$.

With~\Cref{th:fast_rate_mi} at hand, it is clear that all the presented results can be replicated in this setting. Moreover, the choice of the optimal parameter $\lambda$ is simpler since this can be chosen adaptively without resorting to the events' discretization technique from~\citep{rodriguezgalvez2023pacbayes}. 

For completeness, we include the two most important results below. We will present the results in terms of~\Cref{th:fast_rate_mi_single_letter}, where we understand that $\Bar{\kappa_j}$ are defined as above for all $j \in \{ 1, 2, 3 \}$. 

For losses with a bounded $p$-th moment, as far as we are aware, the following is the first result of this kind.

\begin{theorem}
    \label{th:mi_bounded_moments}
    For every loss with a $p$-th moment bounded by $m_p$, the inequality
    \begin{equation*}
        \poprisk \leq \Bar{\kappa}_1 \cdot \emprisk_{\leq t^\star} + m_p^{\frac{1}{p}} \Big(\frac{p}{p-1}\Big) \Big( \Bar{\kappa}_2 \cdot \frac{1}{n} \sum_{i=1}^n \minf(W;Z_i) + \Bar{\kappa}_3 \Big)^{\frac{p-1}{p}}
    \end{equation*}
    holds for all $c_i \in (0,1]$ and all $\gamma_i > 1$, where $$t^\star \coloneqq m_p^{\frac{1}{p}} \Big( \Bar{\kappa}_2 \cdot \frac{1}{n} \sum_{i=1}^n \minf(W;Z_i) + \Bar{\kappa}_3 \Big)^{-\frac{1}{p}}.$$
\end{theorem}

For losses with a bounded variance, the tightest result we know of is from~\citep[Appendix H]{rodriguez2021tighter}, where they show that if the loss has a variance bounded by $\sigma^2$, then 
\begin{equation*}
    \bE[\poprisk(W)] \leq \bE[\emprisk(W,S)] + \frac{1}{n} \sum_{i=1}^n \sqrt{ \sigma^2 \chi^2 (\bP_W^{Z_i}, \bP_W)} 
\end{equation*}
and that
\begin{equation*}
    \bE[\poprisk(W)] \leq \bE[\emprisk(W,S)] +  \sqrt{ \sigma^2 \cdot \frac{\chi^2}{n} }.
\end{equation*}

Similarly to before, the presented~\Cref{th:bounded_variance_mi} improves these results exponentially on the dependence with $\chi^2$ due to the relative entropy bound $\relent \leq \log(1 + \chi^2)$.

\begin{theorem}
    \label{th:bounded_variance_mi}
    For every loss with a variance bounded by $\sigma^2$, the inequality
    \begin{equation*}
        \poprisk \leq \Big[ 1 - 2 \sqrt{\mathfrak{C}_{\mathrm{MI}}} \Big]_+^{-1} \Big[ \Bar{\kappa_1} \cdot \emprisk  + 2 \sqrt{\sigma^2 \mathfrak{C}_{\mathrm{MI}}} \Big]
    \end{equation*}
    holds for all $c_i \in (0,1]$ and all $\gamma_i > 1$, where 
    $$
    \mathfrak{C}_{\mathrm{MI}} \coloneqq \Bar{\kappa_2} \cdot \frac{1}{n} \sum_{i=1}^n \minf(W;Z_i) + \Bar{\kappa}_3.
    $$
\end{theorem}

\subsection{Extension to single-draw PAC-Bayes bounds}
\label{app:extension_single_draw_pac_bayes_bounds}

\looseness=-1 Like in the previous section, we first obtain a ``single-draw PAC-Bayes'' analogue to the fast rate bound from~\citep{rodriguezgalvez2023pacbayes}. Throughout this section, we will define $\mathfrak{C}_{n,\beta,S,W} \coloneqq \log \big( \nicefrac{\rmd \bP_W^S}{\rmd \bQ_W}(W) \big) + \log \big( \nicefrac{\xi(n)}{\beta} \big)$, $\mathfrak{C}_{n,\beta,S,W}' \coloneqq 2 \log \big( \nicefrac{\rmd \bP_W^S}{\rmd \bQ_W}(W) \big) + \log \big( \nicefrac{\pi^2 e^2 \xi(n)}{\beta} \big)$, and $\mathfrak{C}_{n,\beta,S,W}'' \coloneqq \nicefrac{\kappa_2}{n} \cdot  \mathfrak{C}_{n,\beta,S,W}' + \kappa_3$, while understanding that these two are random variables whose randomness comes from the random training set $S$ and the random output hypothesis $W$.

\begin{theorem}
    \label{th:fast_rate_single_draw}
    For every loss with a range bounded in $[0,b]$, with probability no smaller than $1 - \beta$,
    \begin{equation*}
        \poprisk(W) \leq \kappa_1 \cdot  \emprisk(W,S) + b \left[ \kappa_2 \cdot \frac{\mathfrak{C}_{n,\beta,S,W}}{n} + \kappa_3 \right]
    \end{equation*}
    holds for all $c \in (0,1]$ and all $\gamma > 1$, where $\kappa_1 \coloneqq c \gamma \log \big( \nicefrac{\gamma}{(\gamma -1)} \big)$, $\kappa_2 \coloneqq c \gamma$, and $\kappa_3 \coloneqq \gamma \big( 1 - c(1 - \log c)\big)$.
\end{theorem}

\begin{proof}
    Consider \Cref{th:single_draw_general_theorem}, which states that for every measurable function $f: \cW \times \cS \to \bR$, for every $\beta \in (0,1)$, with probability $1- \beta$,
    \begin{equation}
        \label{eq:dv_single_draw}
        f(W,S) \leq \frac{1}{n} \left[ \log \left( \frac{\rmd \bP_W^S}{\rmd \bQ_W}(W) \right) + \log \frac{\Delta}{\beta} \right]
    \end{equation}
    holds, where $\Delta = \bE e^{n f(W',S)}$ and $W'$ is distributed according to the data-independent distribution $\bQ_W$. 

    This theorem is a single-draw version of the Donsker and Varadhan~\citep[Lemma 2.1]{donsker1975asymptotic} and the probability is taken with respect to the draw of the random training set $S$ and the random returned hypothesis $W$. 

    In particular, for every loss with a range bounded in $[0,1]$, considering the convex function 
    $$f(W,S) = \relentber\big( \emprisk(W,S) \Vert \poprisk(W) \big)$$ 
    as in~\citep[Corollary 2.1]{germain2009pac} ensures that $\Delta = \xi(n)$, where $\xi(n)$ is the same as in \eqref{eq:mc_allester}, and then, for every $\beta \in (0,1)$, with probability no smaller than $1 - \beta$,
    \begin{equation}
        \label{eq:dv_single_draw_other}
        \relentber\big( \emprisk(W,S) \Vert \poprisk(W) \big) \leq \frac{\log \left( \frac{\rmd \bP_W^S}{\rmd \bQ_W}(W) \right) + \log \frac{\xi(n)}{\beta}}{n}
    \end{equation}
    holds, where $\relentber( \hat{r} \Vert r) = \relent \big( \textnormal{Ber}(\hat{r}) \Vert \textnormal{Ber}(r) \big)$. \Cref{eq:dv_single_draw_other} is a single-draw version of the Seeger--Langford bound~\citep{seeger2002pac,langford2001bounds}.

    Finally, using the variational representation of the relative entropy borrowed from $f$-divergences~\citep[Theorem 7.24]{polyanskiy2022lecture} and operating like in~\citep[Appendix A.1]{rodriguezgalvez2023pacbayes} completes the proof for losses with a range contained in $[0,1]$. Scaling appropriately extends it to losses with a range contained in $[0,b]$.
\end{proof}

At this point, with~\Cref{th:fast_rate_single_draw}, following the techniques outlined in the main body of the paper to replicate the results for single-draw PAC-Bayes guarantees is routine. The only relevant difference is that to choose the optimal parameter $\lambda$ in \Cref{th:alquier_truncation_method_refined_adaptive_lambda} as outlined in \Cref{app:proof_alquier_truncation_method_refined_adaptive_lambda}, the quantization of the event space is now done with respect to $\log \big( \nicefrac{\rmd \bP_W^S}{\rmd \bQ_W} (W) \big)$ and taking into account that this quantity is random with respect to both the training set $S$ and the hypothesis $W$. That is, the sub-events in the proof are defined as $$\cE_k \coloneqq \left \{ (s,w) \in \cZ^n \times \cW : k-1 \leq \log \Big( \frac{\rmd \bP_W^{S=s}}{\rmd \bQ_W} (w) \Big) \leq k \right \}.$$

The last two result we present in this section, \Cref{th:single_draw_bounded_moments} and \Cref{th:bounded_variance_single_draw} below, are the single-letter (single-draw) extensions of~\Cref{th:alquier_truncation_method_refined_adaptive_lambda} and \Cref{th:bounded_variance_high_probability}, as promised before. Once again, these extensions are enabled by \Cref{th:fast_rate_single_draw}.
To the best of our knowledge, these are also the first single-draw PAC-Bayes results of this kind.

\begin{theorem}
    \label{th:single_draw_bounded_moments}
    For every loss with a $p$-th moment bounded by $m_p$, for all $\beta \in (0,1)$, with probability no smaller than $1 - \beta$,
    \begin{equation*}
        \poprisk(W) \leq \kappa_1 \cdot \emprisk_{\leq t^\star}(W,S) + m_p^{\frac{1}{p}} \Big(\tfrac{p}{p-1}\Big) \Big( \kappa_2 \cdot \frac{\mathfrak{C}_{n,\beta,S,W}'}{n} + \kappa_3 \Big)^{\frac{p-1}{p}}
    \end{equation*}
    holds \emph{simultaneously} for all $c \in (0,1]$ and all $\gamma > 1$, where $$t^\star \coloneqq m_p^{\frac{1}{p}} \Big( \kappa_2 \cdot \frac{\mathfrak{C}_{n,\beta,S,W}'}{n} + \kappa_3 \Big)^{-\frac{1}{p}}.$$
\end{theorem}

\begin{theorem}
    \label{th:bounded_variance_single_draw}
    For every loss with a variance bounded by $\sigma^2$, for all $\beta \in (0,1)$, with probability no smaller than $1 - \beta$,
    \begin{equation*}
        \poprisk \leq \Big[ 1 - 2 \sqrt{\mathfrak{C}_{n,\beta,S,W}''} \Big]_+^{-1} \Big[ \kappa_1 \cdot \emprisk  + 2 \sqrt{\sigma^2 \mathfrak{C}_{n,\beta,S,W}''} \Big]
    \end{equation*}
    holds \emph{simultaneously} for all $c \in (0,1]$ and all $\gamma > 1$.
\end{theorem}

\subsection{Extending Rivasplata's single-draw PAC-Bayesian theorem}

Similarly to~\citet{germain2009pac}'s PAC-Bayesian bound, \citet{rivasplata2020pac}'s single-draw PAC-Bayesian bound requires simultaneously that $\bP_W^S \ll \bQ$ and that $\bQ \ll \bP_W^S$ a.s., since at some point in their proof they use the equality $\nicefrac{\rmd \bP_W^S}{\rmd \bQ} = \big(\nicefrac{\rmd \bQ}{\rmd \bP_W^S}\big)^{-1}$, which only holds when this happens. Similarly to~\citet{begin2014pac}, who lifted the requirement that $\bQ \ll \bP_W^S$ a.s., we present below \citet{rivasplata2020pac}'s result without that extra requirement as well as a simple proof to avoid that requirement.

\begin{theorem}[{Extension of~\citet[Theorem 1 (i)]{rivasplata2020pac}}]
    \label{th:single_draw_general_theorem}
    Consider a measurable function $f: \cW \times \cS \to \bR$. Let $\bQ_W$ be a distribution on $\cW$ independent of $S$ such that $\bP_W^S \ll \bQ_W$ a.s. 
    and $W'$ be a random variable distributed according to $\bQ_W$. Define $\Delta \coloneqq \bE  e^{n f(W',S)}$. Then, for every $\beta \in (0,1)$, with probability no smaller than $1 - \beta$,
    \begin{equation*}
        f(W,S) \leq \frac{1}{n} \left[ \log \left(\frac{\rmd \bP_W^S}{\rmd \bQ}(W) \right) + \log \frac{\Delta}{\beta} \right].
    \end{equation*}
\end{theorem}

\begin{proof}
Consider the non-negative random variable
\begin{equation*}
    X = e^{n f(W,S) - \log \frac{\rmd \bP_W^S}{\rmd \bQ}(W)}.
\end{equation*}
Using a change of measure we have that
\begin{align*}
    &\bE \left[ \bE^S \left[ e^{n f(W,S) - \log \frac{\rmd \bP_W^S}{\rmd \bQ}(W)} \right] \right] \\
    & \quad= \bE \left[ \bE^S \left[ e^{n f(W',S) - \log \frac{\rmd \bP_W^S}{\rmd \bQ}(W')} \cdot \frac{\rmd \bP_W^S}{\rmd \bQ}(W') \right] \right] \\
    & \quad = \bE \left[ e^{n f(W',S)} \right].
\end{align*}
Then, applying Markov's inequality to the random variable $X$ we have that
\begin{align*}
    \bP \left[ e^{n f(W,S) - \log \frac{\rmd \bP_W^S}{\rmd \bQ}(W)} \geq \frac{1}{\beta} \cdot  \bE \left[ e^{n f(W',S)} \right]\right] \leq \beta.
\end{align*}

Since the logarithm is a non-decreasing, monotonic function we can take the logarithm to both sides of the inequality and re-arrange the terms to obtain the desired result.
\end{proof}

\end{document}